\documentclass{article}
\usepackage{ijcai20}

\usepackage{times}
\usepackage{soul}
\usepackage{url}
\usepackage[hidelinks]{hyperref}
\usepackage[utf8]{inputenc}
\usepackage[small]{caption}
\usepackage{graphicx}
\usepackage{amsmath}
\usepackage{amsthm}
\usepackage{booktabs}
\usepackage{algorithm}
\usepackage{algorithmic}
\usepackage{subcaption}

\usepackage{multirow}

\usepackage{amsmath}
\usepackage{amssymb}
\usepackage{mathrsfs}
\usepackage{amsmath,amssymb,amsthm}
\newtheorem{thm}{Theorem}
\newtheorem{definition}{Definition}

\title{Redundancy of Hidden Layers in Deep Learning: An Information Perspective}

\author{
Chenguang Zhang$^1$\and
Yuexian Hou$^{*1}$\and
Dawei Song$^{2}$\and
Liangzhu Ge$^1$ \and
Shuai Yao$^1$
\affiliations
$^1$College of Intelligence and Computing(School of Computer Science and Technology), Tianjin University\\
$^2$Mathematics College of Intelligence and Computing, The Open University\\
\emails
Corresponding author: yxhou@tju.edu.cn
}

\begin{document}

\maketitle

\begin{abstract}
Although deep structures guarantee powerful expressivity of deep neural networks (DNNs), they may also trigger overfitting problems. To improve the generalization capability of DNNs while retaining their expressivity, many strategies were developed to improve the diversity among the hidden units. Following this research direction, we propose a label-based diversity measure (LDiversity) quantified as the gap between a newly added inductive-bias term and a canonical unsupervised diversity measure term by formalizing the effect of the entanglement of the hidden units on the generalization capacity as mutual information. The existence of an inverse relationship between LDiversity and the generalization capacity is proved; 
i.e., the decrease in LDiversity generally improves the generalization capacity. Further, a regularization method is proposed by using LDiversity as the regularizer. The experiments show that the new method can effectively reduce overfitting and decrease the generalization error, experimentally justifying our approach.
\end{abstract}

\section{Introduction}
\label{Intro}

Deep neural networks (DNNs) have achieved significant
success in many practical applications due to their strong
expression capacity and powerful learning ability. 
However,
the deep structure of DNNs may lead to complicated nonlinear
mappings from input to output, giving rise to the problem of overfitting. Therefore, many studies have been devoted to
developing approaches to improve the generalization capacity of DNNs in order to
address the overfitting problem. One research direction is to constrain the model complexity; these methods include dropout \cite{srivastava2014dropout}, weight decay \cite{krogh1991a}, and CIF \cite{zhao2018a}. 
Despite the high effectiveness  of these methods, they may not fully leverage the expression capability  of the model because they generally reduce the effective number of model parameters.

To improve the generalization capacity of DNNs while  simultaneously maintaining their expressivity, another research direction to pursue is to explore the diversity among the hidden units of a single specified layer of DNNs and to improve it, encouraging hidden units to be as uncorrelated or independent from each other as possible. 
For instance, Cogswell et al. minimized the cross-covariance of hidden activations to obtain diverse representation of hidden units to reduce overfitting \cite{cogswell2016reducing}. Gu et al. extended this method, treating nonoverlapping groups of hidden units as component learners in order to avoid the negative influence of the breakdown of correlations \cite{gu2018regularizing}. Impressively, by using mutual information among the hidden units as the measure of diversity, Brakel et al. proposed a method to learn independent features \cite{brakel2018learning}. Many other studies reported in the literature also investigated the positive role of feature independence in the feature encoding \cite{bengio2017independently,hjelm2019learning}. 

This paper follows the second research direction and further focuses on the role of label information in defining the diversity measure. In fact, while boosting the diversity among the hidden units of DNNs has been shown to be beneficial for the performance in classification tasks, there is no single widely accepted formal definition of diversity measure. It was found  that the role of inductive biases should be made explicit and enforced in the process of learning disentangled hidden representations for downstream tasks \cite{locatello2019challenging}. However, the measures of diversity in the current set mostly only considered the correlations among the hidden units over the whole mixed data distribution but neglected the local clustering feature of different classes, 
which may be harmful to the classification performance \cite{grover2019uncertainty}. 
A feasible approach to understand the role of label information in defining the diversity measure and obtain an appropriate measure with embedded inductive biases is to investigate the measure of diversity under supervised settings by examining the relationship between diversity and generalization capacity.  

  
To achieve this goal, we first introduce a generalization error bound of DNNs from an information perspective, following the work of Xu et al. \cite{xu2017information}, which formalizes the bound as mutual information between the activation values of the hidden units in the specified layer and model parameters from this layer to the end layer (see Fig.\ref{F1}). Intuitively, the new bound describes the information of the extracted features stored in the model parameters, which was proposed as a measure of the effective complexity of a network by Hinton and Van Camp and was used as a regularizer to simplify the networks by Achille and Soatto \cite{hinton1993keeping,achille2018emergence}. However, the direct usage of this bound is usually difficult due to its excessively complicated estimation in previous work. Nevertheless, compared to the traditional generalization error bounds that are based on hypothesis space, e.g., the Rademacher complexity \cite{boucheron2005theory} and the uniform stability \cite{bousquet2002stability}, the new bound is tight and simpler in form, making it sufficient for our further analysis.  

\begin{figure}[t]
\centering
\includegraphics[width=0.4\textwidth, height=0.18\textwidth]{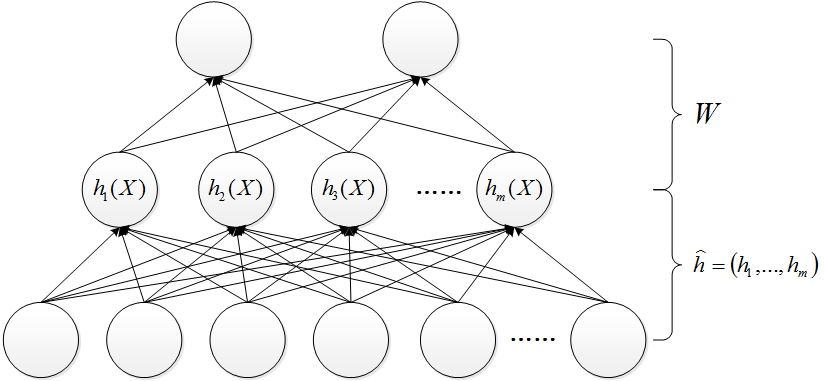}
\caption{ The hidden units are shown as mappings; $W$ is the collection of model parameters from the specified layer to the end layer.}
\label{F1}
\end{figure}

We then decompose the new bound and naturally obtain  a measure of diversity formulated as a difference of two terms, named the label-based diversity measure (LDiversity), where the first term that has been used to obtain independent features \cite{brakel2018learning} is a canonical unsupervised diversity measure and the second term is a label-based term that has  not been considered by the current diversity measures in DNNs;  this term is the main difference between our measure and the other measures, reflecting the inductive biases embedded in the hidden representations. Decreasing the defined label-based diversity measure is expected to improve the generalization capacity. 
Furthermore, if regarding the hidden units as base learners, the proposed diversity measure can also be viewed as the ensemble diversity proposed by Brown et al. \cite{brown2009an,zhou2010multi} for ensemble learning, where the ensemble diversity is shown to be the part of the upper bound of classification error. This means that decreasing LDiversity may also suppress the probability of classification error.
   
By using LDiversity as the regularizer, we develop a new regularization method named the LDiversity method (LDM), the goal of which is to minimize the classification loss and newly added LDiversity. In particular, the process of minimizing LDiversity is the same as that of training generative adversarial networks (GANs), where two additional ``discriminators'' are involved to estimate LDiversity by maximizing their output values. Finally, we apply this method to fully connected neural networks and convolutional neural networks. Many experiments show that LDM can effectively reduce the overfitting and decrease the generalization error compared to the methods without the LDiversity regularizer. Furthermore, LDiversity between hidden units is demonstrated to be a crucial factor for reducing the generalization error in DNNs.

\section{Preliminaries}
 Let $\mathcal{Z} = (\mathcal{X}, \mathcal{Y})$ be an instance space, where $\mathcal{X}$ is a feature space and $\mathcal{Y}$ is a label space. A training set $S$ of size $n$ is an $n$-tuple, i.e.,
\begin{equation}\label{eq5}
S=(Z^1, Z^2, ..., Z^n)
\end{equation}
of i.i.d random elements of $Z=(X, Y)\in\mathcal{Z}$ with an unknown PDF $P_Z(z)$. 
Given a neural network with multilayers, let $\hat{h}= (h_1, h_2,...,h_m)$ be the set of all hidden units $h_i$ in the discussed layer and $W\in \mathcal{W}$ be the collection of model parameters from the specified hidden layer to the end layer, where $\mathcal{W}$ is the hypothesis space of $W$. Due to the randomness in the realization of the dataset $S$, the values in 
$W$ are considered to be random variables.      

We will make frequent use of the following standard information theoretical quantities \cite{cover1991elements}. For a stochastic variable $X$, its Shannon entropy is defined as
\begin{equation}
\label{eq1}
H(X) = -\mathbb{E}_X[\log P_X(X)],
\end{equation}
where  $\mathbb{E}_{X}[\cdot]$ denotes the expectation of the random object within the brackets w.r.t to the subscript random variable $X$.

The mutual information of two stochastic variables is
\begin{equation}\label{eq2}
I(h_1; h_2) = H(h_1) + H(h_2) - H(h_1, h_2) 
\end{equation}
which by capturing the nonlinear statistical dependencies between the  variables can be reformulated
as the Kullback-Leibler (KL-)divergence between the
joint density and the product of the marginal densities, i.e.,
\begin{equation}\label{eq3}
I(h_{1}; h_{2})=D_{KL}(P_{h_1h_2}||P_{h_1}\otimes P_{h_2}),
\end{equation}
which is zero if and only if $h_1$ and $h_2$ are independent.
For more than two variables, the multivariate mutual information is defined as
\begin{equation}\label{eq4}
\begin{aligned}
&I(h_1; h_2,...;h_n)=D_{KL}(P_{h_1h_2...h_n}||\otimes_{i=1}^{m} P_{h_i}).\\
\end{aligned}
\end{equation}
We will later use it to measure the entanglement in the hidden units and still call it the mutual information in the following for consistency. Consequently, the conditional mutual information of multiple variables given $Y$ is 
\begin{equation}\label{eq5}
I(h_1; h_2,...;h_n|Y) = \mathbb{E}_Y[D_{KL}(P_{h_1h_2...h_n}||\otimes_{i=1}^{m} P_{h_i}|Y)],
\end{equation}
and will be employed below to describe the class-conditional correlation.

\section{Generalization Error Bound}
This section gives an upper-bound of the generalization error from an information-theoretical perspective, formulating the absolute difference in the expectation between the expected risk and the empirical risk as the mutual information between the hidden units and the model parameters.  
We follow the framework proposed by Russo and Xu et al. \cite{russo2020how,xu2017information}, where it is convenient to think of each unit in the discussed layer as a mapping from the datum to its activation value, i.e., $h_i: \mathcal{X} \to \mathbb{R} (i=1,...,m)$, to investigate the effect of entanglement of the hidden units on the generalization capacity since both the expected risk and the empirical risk are initially related to the datum rather than to the values of the hidden units; in fact,  we are more interested in the mechanism that produces the activation value than in the activation value itself (see Fig. \ref{F1}). Consequently, let $\hat{h}(X) = (h_1(X), ..., h_m(X))$ , $f(Z) = (\hat{h}(X), Y)$; then, given $W\in \mathcal{W}$, the loss function $l$ on the sample $Z$ can be restated as a function w.r.t. $f(Z)$ and $W$, i.e., $l:\mathcal{W}\times f(\mathcal{Z})\to \mathbb{R}^{+}$, where $Z=(X, Y)$. Accordingly, let $f(S) = \left(f(Z^1),..., f(Z^n)\right)$. Now, we are ready to obtain the upper bound.

The empirical risk of a hypothesis $W\in \mathcal{W}$ over the dataset $S$ is
\begin{equation}\label{eq6}
L_{{f(S)}}(W) \triangleq \frac{1}{n}\sum_{i=1}^{n}{l(f(Z^i), W)}.
\end{equation}
The expected risk of $W$ on $P_{S}$ is 
\begin{equation}\label{eq7}
\begin{aligned}
&L_{\overline{f(S)}}(W)\triangleq \mathbb{E}_{S} \left[ \frac{1}{n}\sum_{i=1}^{n}{l(f(Z^i), W)}\right]\\
&=\mathbb{E}_{f(S)}\left[\frac{1}{n}\sum_{i=1}^{n}{l(F^i, W)}\right],
\end{aligned}
\end{equation}
where $F^i=f(Z^i)$ $(1 \leq i \leq n)$ are i.i.d random variables. 
Taking expectation on the difference between $L_{f(S)}(W)$ and $L_{\overline{f(S)}}(W)$ with respect to the joint distribution $P_{(S,W)}(s, w)$, we obtain
\begin{equation}\label{eq9}
g(P_S,P_{W|S}) \triangleq \mathbb{E}_{(S,W)}\left[L_{\overline{f(S)}}(W) -L_{{f(S)}}(W)\right]. 
\end{equation}
Then, the generalization error can be decomposed as
\begin{equation}
\mathbb{E}\left[L_{\overline{f(S)}}(W) \right] = \mathbb{E}\left[L_{{f(S)}}(W) \right] + g(P_S,P_{W|S}). 
\end{equation}
We focus on $g(P_S,P_{W|S})$, which reflects the quality of the generalization of the output hypothesis. Some further steps show that
\begin{equation}\label{eq8}
\resizebox{.99\linewidth}{!}{$
\begin{aligned}
&g(P_S,P_{W|S}) =\mathbb{E}_{(f(S),W)}\bigl[ \mathbb{E}_{f(S)}\left[L_{f(S)}(W)\right] -L_{f(S)}(W)\bigr] \\
&=\mathbb{E}_{f(S)\otimes W}\left[L_{f(S)}(W)\right] - \mathbb{E}_{(f(S), W)}\left[L_{f(S)}(W)\right]
\end{aligned}
$}
\end{equation}
where $\mathbb{E}_{f(S)\otimes W}$ means taking the expectation w.r.t the product of the marginal PDFs of $f(S)$ and $W$.

Xu and Raginsky (Lemma 1 in \cite{xu2017information}) have justified that given two random variables $X$ and $Y$ with the joint PDF $P_{XY}$ and the product of the marginal PDFs $P_{\overline{X}\overline{Y}}=P_X\otimes P_Y$, if the function $c(X, Y)$ is a $\sigma-$subgaussian function under $P_{\overline{X}\overline{Y}}$, then 
\begin{equation}\label{eq9}
\left|\mathbb{E}_{(X,Y)}\left[c(X,Y)\right]-\mathbb{E}_{X\otimes Y}\left[c(X,Y)\right]\right|\leq\sqrt{2\sigma^2I(X; Y)}.
\end{equation}
where a random variable $U$ is $\sigma-$subguassian if $\log\mathbb{E}\left[e^{\gamma(U-\mathbb{E}[U])}\right]\leq\gamma^2\sigma^2/2$ for all $\gamma \in \mathbb{R}$. In fact, if the loss function in Eqs. \eqref{eq6} and \eqref{eq7} is restricted as a function bounded in $[a, b]$, e.g., the sigmoid function, thereby being a $\sigma$-subgaussian function by Hoeffding's lemma \cite{massart2007concentration}, then $L_{f(S)}(W)$ in Eq. \eqref{eq8} is consequently a $\sigma/\sqrt{n}$-subgaussian function for $W$ due to the independence among $f(Z_i) (1\leq i \leq n)$, where $\sigma = (b-a)/2$.
Then, according to Eq. \eqref{eq9}, by setting $X$ and $Y$ in Eq. \eqref{eq9} as $f(S)$ and $W$, respectively, we obtain the following lemma.
\newtheorem{Lemma}{Lemma}
\begin{Lemma}\label{lemma1}
If the loss function $l$ is $\sigma$-subgaussian, then the absolute value of $g(P_S, P_{W|S})$ is upper-bounded in terms of the mutual information between $f(S)$ and $W$, i.e.,
\begin{equation}\label{eq10}
\left|g(P_S, P_{W|S})\right| \leq \frac{1}{n} \sqrt{2\sigma^2}\sqrt{nI(f(S); W)}.
\end{equation}
\end{Lemma}
\section{Label-based Diversity Measure}

The generalization error bound deduced by Lemma 1 may be more tightly coupled to the generalization error than some existing bounds can because the new bound depends on almost all ingredients of learning problems, including the distribution of the dataset, the hypothesis space and the learning algorithm, while some existing bounds, such as VC dimension or Rademacher complexity, mainly depend on the hypothesis space and neglect of the learning algorithm, which potentially means a looser quantity of bound to unify the other ignored ingredients in the learning problems. Moreover, Lemma 1 implies that regularizing the empirical risk with $I(f(S); W)$ may lead to improved genelarization capacity. However, due to the high dimensions of hypothesis space $\mathcal{W}$, the direct usage of $I(f(S); W)$ is usually intractable. In this section, we decompose the upper bound in Eq. \eqref{eq10}, remove the terms related to $W$ and naturally derive a label-based diversity measure among the hidden units.

 \begin{thm}\label{lemma2}
Decomposing the upper bound in Eq. \eqref{eq10}, we obtain
\begin{equation}\label{eq11}
\resizebox{.99\linewidth}{!}{$
\begin{aligned}
& \frac{1}{n} \sqrt{nI(f(S); W)}=\\
& \sqrt{
                        \begin{split}
                              &I(h_1(X);...;h_m(X))-I(h_1(X);...;h_m(X)|Y)\\
                              &-\sum_{i=1}^{m}{I\bigl(h_i(X); Y\bigr)}+H(Y)+\frac{H(W)-H(S_y, W|\hat{h}(S_x))}{n},
                        \end{split}  
                         }
\end{aligned}
$}
\end{equation}
where $S_x = \{X_1,...,X_n\}$, $S_y = \{Y_1,...,Y_n\}$ and $\hat{h}(S_x) = \{{h}(X_1),...,{h}(X_n)\}$.
\end{thm}
\begin{proof}
Only a belief proof is given here (see the the supplementary material for the details).
\begin{equation}\label{eq12}
\begin{aligned}
&I(f(S);W)=I(\hat{h}(S_x), S_y; W)\\
&=H(h_1(S_x),...,h_n(S_x), S_y) \\
&\quad - H(h_1(S_x),...,h_n(S_x), S_y, W) + H(W)\\
\end{aligned}
\end{equation}
By adding $ \sum_{i=1}^{m}{H(h_i(S_x))} - \sum_{i=1}^{m}{H(h_i(S_x)|S_y)} - \sum_{i=1}^{m}{I(h_i(S_x), S_y)} $ to the right-hand side of the above equation,
it follows that
\begin{equation}\label{eq13}
\begin{aligned}
&I(f(S);W)=  - \sum_{i=1}^{m}{I(h_i(S_x); S_y)}+ H(W)\\
&+\big[H(S_y)- I(h_1(S_x);...;h_m(S_x)|S_y)\big]\\
& +\big[I\big(h_1(S_x);...;h_m(S_x)\big) - H\big(S_y,W|\hat h{(S_x)}\big)\big] .
\end{aligned}
\end{equation}
Considering that the samples $(X_i, Y_i) $ in $S$ are sampled in an i.i.d.
fashion, we have 
\begin{equation}
\begin{aligned}
&I(h_1(S_x);h_2(S_x);...;h_m(S_x)|S_y)\\
&\,\,\,\,=nI(h_1(X);h_2(X);...,h_m(X)|Y),\\
&I(h_1(S_x);...;h_m(S_x))=nI(h_1(X);...;h_m(X)),\\
&\sum_{i=1}^{m}{I(S_y;h_i(S_x))} + H(S_y)\\
&\,\,\,\,=n\big(\sum_{i=1}^{m}{I(h_i(X);Y)}+H(Y)\big).\\
\end{aligned}
\end{equation}
Combining above equations gives Eq. \eqref{eq11}, which completes the proof.
\end{proof}

Let us focus on the Eq. \eqref{eq11}. There are five terms in the square root. Only the first two terms are completely unrelated to the sample size $n$ and the model parameters $W$, reflecting the relationships among the hidden units. Since the two terms are part of the decomposed upper bound, regularizing the empirical risk with their sum is expected to reduce the upper bound of the absolute value of $g(P_S, P_{W|S})$ as well as the generalization error. We argue that the two terms naturally quantify the diversity among the hidden units (see Definition 1). 
For remaining terms: the third term, which is the sum of the respective relevancy of the hidden units $h_i (1\leq i \leq m)$ to the labels $Y$, demonstrating the classification ability of each hidden unit itself and as a whole having a positive correlation, with the second term to some extent, is not considered by our diversity measure; the forth term $H(Y) $ is non-optimizable w.r.t. the training process, which is also not considered by the new diversity measure; Compared to the other terms, the last term is the only one related to the sample size $n$ and the model parameters $W$, which makes it unsuitable for describing the diversity among hidden units. Moreover, when the sample size is relatively large, this term tends to be relatively small and consequently has a small effect on the generalization error since the upper bound in Lemma 1 for any reasonable hypothesis necessarily declines to a very small value when the sample size increases. Thus, we do not consider optimizing this term in this work.

\begin{definition}
The label-based diversity measure among the hidden units is defined as
\begin{equation}\label{eq17}
\resizebox{.99\linewidth}{!}{$
D_{LB}(S) = I(h_1(X);...;h_m(X)) - I(h_1(X);...;h_m(X)|Y).
$}
\end{equation}
\end{definition}
As discussed in the introduction, the first term in the diversity measure is the canonical unsupervised diversity measure, which was used to learn independent data representations \cite{brakel2018learning}. It has also been shown that reducing such a term of correlations among the hidden units will lead to an improved generalization capability\cite{cogswell2016reducing,hjelm2019learning}. The second term in LDiversity is a label-dependent term; it describes the local clustering feature captured by the hidden units. Improving this term may strengthen the class-conditional correlation and make the activations of the same class behave more collaboratively, which is usually important for a classification task. This term is the main difference between  our diversity measure and other diversity measures.
 
It is worth noting that the diversity measure is identical in form to the ensemble diversity measure proposed by \cite{brown2009an,zhou2010multi} for ensemble learning. They derived the ensemble measure by analyzing the upper bound of the probability of the classification error,which is \cite{hellman1970probability} 
\begin{equation}\label{eq181}
P(C(\hat h)\neq Y) \leq \frac{H(Y) -I(\hat h; Y)}{2},
\end{equation}
where we continue to use the previous symbols and see $\hat h = (h_1, h_2, ..., h_m)$ as a set of base classifiers for the sample $(X, Y)$; $C(\cdot)$ is any given combination function that minimizes the probatility $P(C(\hat h)\neq Y)$. By decomposing the mutual information term in Eq. \eqref{eq181}, they obtained
\begin{equation}\label{eq18}
\resizebox{0.99\linewidth}{!}{$
\begin{aligned}
I(\hat h; Y) &= \sum\limits_{i=1}^{m}{I\bigl(h_i; Y\bigr)} + I(h_1;...;h_m|Y)- I(h_1;...;h_m),
\end{aligned}
$}
\end{equation}
and defined the sum of the last two terms as the ensemble diversity measure. Although the two measures have the same form, the ensemble diversity measure is based on the the Bayesian learning framework where only 0-1 loss is permitted, which makes it not very suitable for the cases in deep learning; they also did not propose an effective process for using the measure in practice. Nonetheless, their work implies that if we can regard the hidden units as base classifers, the decrease in  LDiversity may well lead to a decrease of the classification error. 

\section{Regularization Method}
In this section, a new regularization method named the label-based diversity method (LDM) is proposed by using LDiversity
 $D_{LB}$ as the regularizer. 
Its total loss function is formulated as  
\begin{equation}
\label{eq29}
T_{loss} = E_{loss} + \lambda D_{LB},
\end{equation}
where $E_{loss}$ is the premier loss function of the DNNs without any regularizer, for instance, the cross-entropy between the outputs of DNNs and the labels; $D_{LB}$ controls the label-based diversity among the hidden units in one specified layer; and $\lambda > 0$ is the balance parameter.

The regularizer $D_{LB}$ is actually the difference of two mutual information terms. Although the estimation of mutual information was recognized as a very difficult problem due to the continuity and high dimensions of data, recent studies \cite{belghazi2018mine,brakel2018learning} revealed that this problem can be solved in terms of the Donsker-Varadhan representation \cite{donsker1975asymptotic} of KL-based mutual information, which is
\begin{equation}
D_{KL}(P||Q) = \sup_{T:\Omega \to \mathbb R}\left\{\mathbb{E}_P[T] - \log(\mathbb{E}_Q[e^T])\right\},
\end{equation}
where $T$ is usually realized as a neural network such that the two expectations are finite. Then, by Eq. \eqref{eq4}, the mutual information is estimated by optimizing $T$ to narrow the divergence between the joint distribution and the product of the marginals. However, such strategy for KL-based mutual information may suffer from the instability problem; an alternative approach is to use Jensen-Shannon (JS-)divergence-based mutual information to replace KL-based mutual information as proposed by Brakel and Bengio, where the possible deviation of using JS-divergence is usually acceptable \cite{brakel2018learning}. That is,
\begin{equation}\label{eq21}
I(h_{1}; h_{2};...;h_{m})\approx D_{JS}(P_{h_1h_2...h_m}||\otimes_{i=1}^{m}P_{h_i}),
\end{equation}
where $D_{JS}$ is JS-divergence. It is estimated by
\begin{equation}\label{eq22}
\begin{aligned}
&D_{JS}(P_{h_1h_2...h_m}||\otimes_{i=1}^{m}P_{h_i})\geq\sup_{T}\bigl\{\mathbb{E}_{}\bigl[\log[\sigma(T(\hat{h}))]\bigr]\\
&\, +\mathbb{E}\bigl[\log[1-\sigma(T(\overline{h}))]\bigr]\bigr\},
\end{aligned}
\end{equation}
where $\overline{h} = (\overline{h_1}, \overline{h_2},...,\overline{h_m})$ obeys the distribution $\otimes_{i=1}^{m}P_{h_i}$ and hereinafter $\sigma(\cdot)$ represents the sigmoid function. Similarly, for the conditional mutual information, we have
\begin{equation}\label{eq23}
I(h_{1}; h_{2};...;h_{m}|Y)\approx \mathbb{E}_Y\left[D_{JS}(P_{h_1h_2...h_m}||\otimes_{i=1}^{m}P_{h_i}|Y)\right],
\end{equation}
where taking expectation on $Y$ requires using Eq. \eqref{eq22} to first obtain the corresponding JS-divergence for any given $Y$ and combining the obtained divergence according to the prior probability of $Y$, which is estimated by the proportion of samples of each class to the total. To distinguish the network $T$ used in Eqs. \eqref{eq21} and \eqref{eq23}, they are denoted by $T_1$ and  $T_2$, respectively. For brevity, the JS-divergence for conditional mutual information is abbreviated as $D_{JS}^L$.

The learning algorithm is finally shown as an iterative min-max process: 
\begin{equation}
\label{eq24}
\resizebox{.99\linewidth}{!}{$
\begin{aligned}
&\min_{\theta}\biggl\{E_{loss} + \lambda\max_{T_1}\biggl[\mathbb{E}_{\hat{h}}\bigl[\log[\sigma(T_1)]\bigr]+\mathbb{E}_{\overline{h}}\bigl[\log[1-\sigma(T_1)]\bigr]\biggr]\\
&- \lambda\mathbb{E}_{Y}\biggl\{\max_{T_2}\biggl[\mathbb{E}_{\hat{h}|Y}\bigl[\log[\sigma(T_2)]\bigr]+\mathbb{E}_{\overline{h}|Y}\bigl[\log[1-\sigma(T_2)]\bigr]\biggr]\biggr\}\biggr\},
\end{aligned}
$}
\end{equation}
where the maximization process guarantees a sufficient approximation to LDiversity $D_{LB}$ by $T_1$ and $T_2$; and the minimization process is the training process of the regularizing DNNs to obtain the classifier $\theta$. The overall training process of LDM is similar to that of generative adversarial networks (GANs) \cite{goodfellow2014generative}. In fact, both $T_1$ and $T_2$ in LDM play the same role as the discriminator in GANs (see Fig. \ref{F2}).

\begin{figure}[tb]
\centering
\includegraphics[width=0.4\textwidth, height=0.25\textwidth]{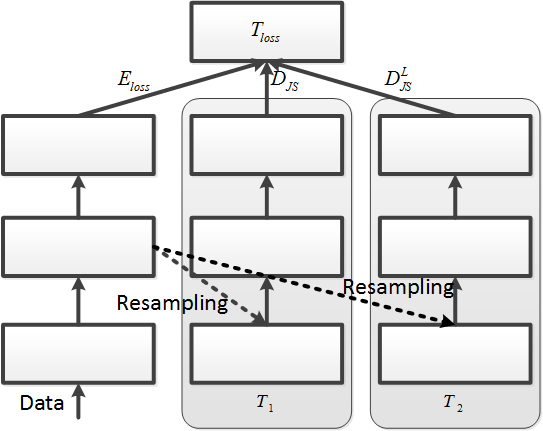}
\caption{ Architecture of the proposed method LDM}
\label{F2}
\end{figure}

To implement the learning algorithm presented in Eq. \eqref{eq24}, it is important to estimate the expectations first. The expectation taken on the joint distribution $P_{\hat{h}}$ or $P_{\hat{h}|Y}$, can be estimated directly by its average value on the samples from the joint distribution. However, taking expectation on the product of marginals  $P_{\overline{h}}$ or $P_{\overline{h}|Y}$ is not straightforward because there are no samples from such a distribution for the empirical estimation. In this work, two strategies are established to approximately obtain samples from the product of marginals according to the type of network. For the case of fully connected neural networks, each sample from the product of marginals with $M$ dimensions is obtained by randomly selecting $M$ samples from the joint distribution, taking the $i$th element from the $i$th selected sample and combining them. For convolutional neural networks (CNNs), after viewing each filter in the discussed layer as a map, each group of the mapped values of all the filters is seen as a sample from the joint distribution. For instance, in the case where there are 3 filters in the the specified CNNs layer, then the samples from the joint distribution are the vectors with 3 mapped values of all the 3 filters. Then, by the same method applied to the fully connected neural networks, we can obtain the samples from the product of marginals. The implementation of LDM is presented in Algorithm 1 (also see Fig. \ref{F2}).

\begin{algorithm}[t]
\caption{Adversarial train loop of LDM}
\label{alg:algorithm}
\textbf{Input}: dataset $S$, classifier $\theta$ as well as auxiliary networks $T_1$ and $T_2$, hidden units $\hat{h}=(h_1,h_2,...,h_m)$ in one specified layer as maps, loss function $l(\cdot)$\\
\textbf{Parameter}: $\lambda$\\
\textbf{Output}: $\theta$
\begin{algorithmic}[1] 
\WHILE{Not converged}
\STATE Sample a batch ${S_b}$ of K data vectors from $S$. 
\STATE $\hat{h}_{{S_b}}\leftarrow \hat{h}(S_b)$, $E_{loss}\leftarrow l(S_b, \theta)$. 
\STATE Resample from $\hat{h}_{{S_b}}$ to obtain new samples $\overline{h}_{{S_b}}$ that obey the product of the marginals. 
\STATE $D_{JS} \leftarrow \log[\sigma(T_1(\hat{h}_{{S_b}}))]+\log[1-\sigma(T_1(\overline{h}_{{S_b}}))]$. 
\STATE Update $T_1$ to maximize $D_{JS}$.
\STATE $D_{JS}^L\leftarrow 0$ . 
\FOR {each label $i$} 
\STATE Get all samples ${{S_{i}}}$ in $S_b$ with label $i$
\STATE $\hat{h}_{{S_i}}\leftarrow \hat{h}(S_{i})$. 
\STATE Resample from $\hat{h}_{{S_i}}$ to obtain new samples $\overline{h}_{{S_i}}$ that obey the product of the marginals. 
\STATE $D_{JS}^{L_i} \leftarrow \log[\sigma(T_2(\hat{h}_{{S_i}}))]+\log[1-\sigma(T_2(\overline{h}_{{S_i}}))]$. 
\STATE Update $T_2$ to maximize $D_{JS}^{L_i}$.
\STATE Estimate $P(Y=i)$.
\STATE $D_{JS}^{L} \leftarrow D_{JS}^{L} + P(Y=i) \cdot D_{JS}^{L_i}$.
\ENDFOR
\STATE $T_{loss} = E_{loss} + \lambda (D_{JS} - D_{JS}^L)$.
\STATE Update $\theta$ to minimize $T_{loss}$.
\ENDWHILE
\STATE \textbf{return} $\theta$
\end{algorithmic}
\end{algorithm}

\section{Experiments}
We apply LDM to fully connected neural networks and convolutional neural networks. We compare it with the method without a regularizer (NONE), dropout with a dropout rate of 0.5 \cite{srivastava2014dropout}, the method with decorrelation regularizer (Decov) in which the hyperparameter was set to 0.1 \cite{cogswell2016reducing}
and the method with the unsupervised diversity term in LDiversity as a regularizer (UDM), in which the balance parameter was set to 0.1 as proposed by Brakel and Bengio in \cite{brakel2018learning}.
All the methods involved were implemented by TensorFlow \cite{abadi2016tensorflow}.
\subsection{Experiments on Fully Connected Neural Networks}
\paragraph{Dataset.} The experiments were conducted on the MNIST dataset \cite{lecun2010mnist}, which contains a training set of 60000 samples and a test set of 10000 samples with pixel values normalized to [0, 1]. Moreover, Gaussian noise with a mean value 0 and variance 1 was added to the original dataset to increase the performance differentiation.  
\paragraph{Method Settings.} 
Since the goal is to check the role of the inductive-bias term in LDiversity and evaluate the performance of LDM and other regularization methods, we used a simple 3-layer fully connected network in this work, with 32 ReLUs in the hidden layer and ten units in the output layer. The batch size was set to 64. The main network was trained using the Adam algorithm \cite{da2014method} with a learning rate of 0.001 until the number of iterations reached 1000. Moreover, the architectures of two auxiliary networks $T_1$ and  $T_2$ were both set to 32-200-1. Their training settings were the same as those of the main network except that the number of updates for each update of main network was set to 4; and the learning rate was set to 0.0001. The balance parameter $\lambda$ was set to 0.7.  

\begin{figure*}[t]
\centering
\begin{subfigure}{0.3\linewidth}
 \label{Fig1a}\includegraphics[width=\columnwidth,  height=\textwidth]{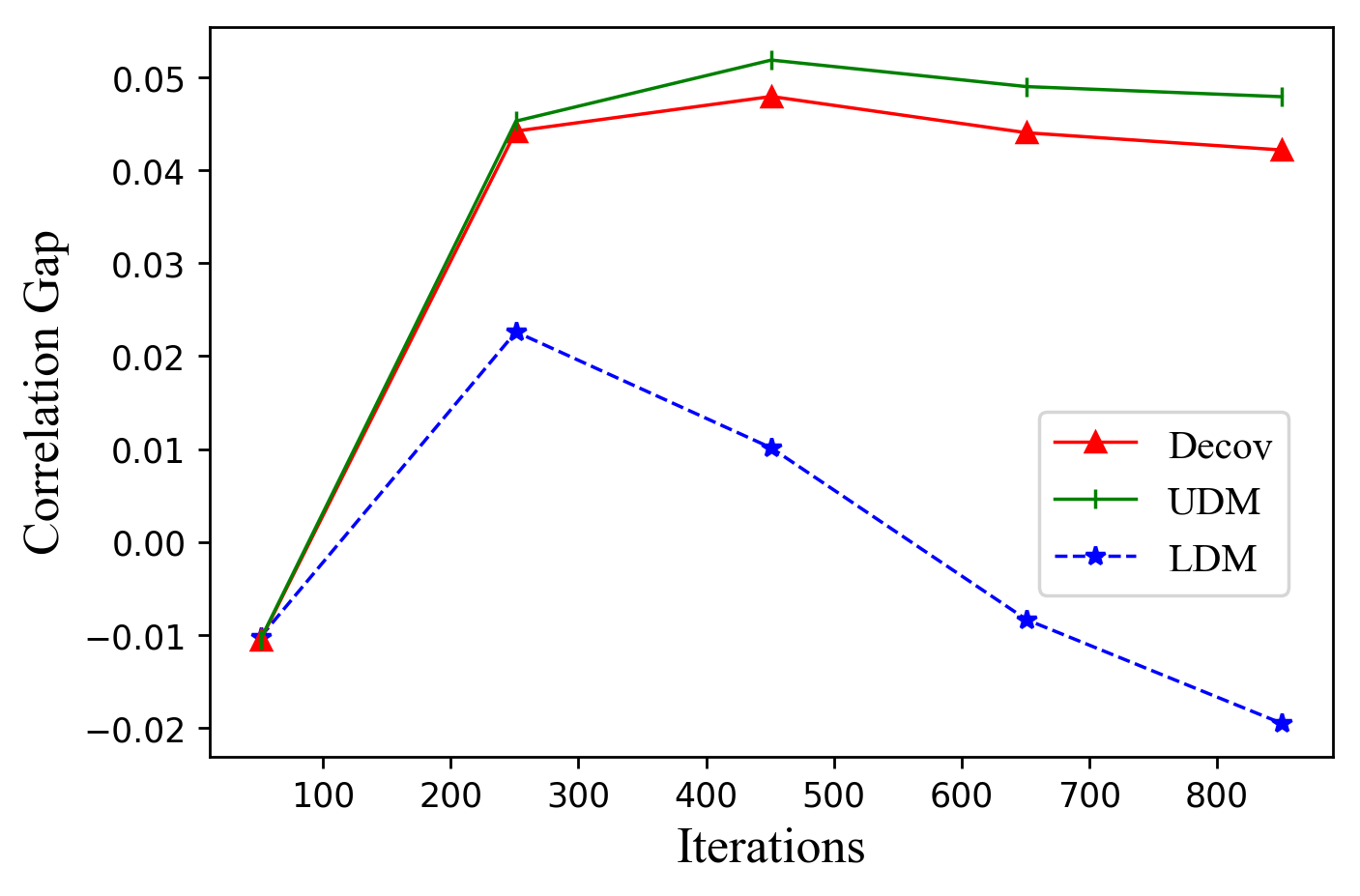} 
\caption{MNIST} 
\end{subfigure} 
\begin{subfigure}{0.3\linewidth}
\label{Fig1b} 
\includegraphics[width=\columnwidth, height=\textwidth]{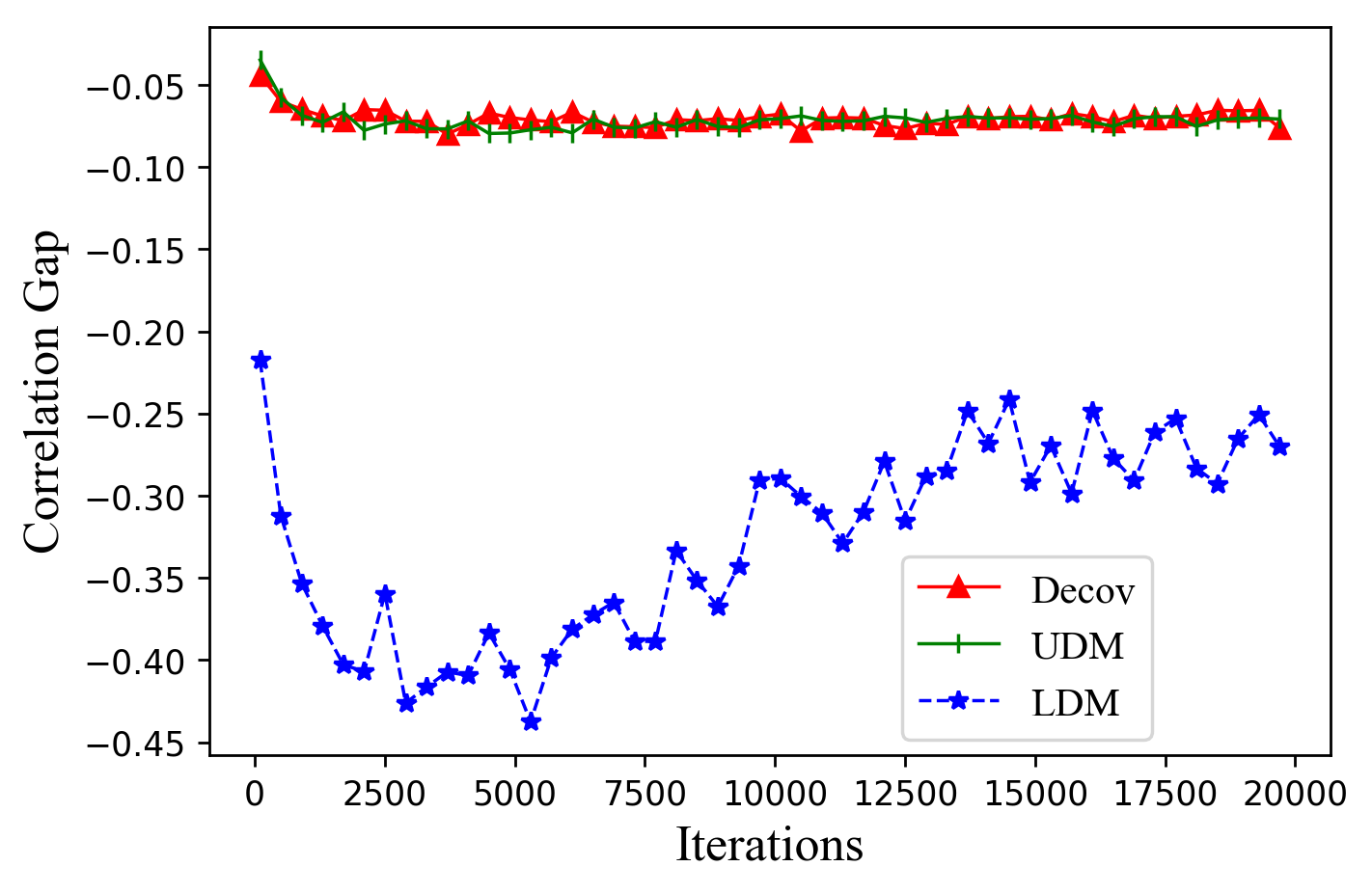}
\caption{CIFAR-10} 
\end{subfigure}
\begin{subfigure}{0.3\linewidth}
 \label{Fig1c} 
\includegraphics[width=\columnwidth, height=\textwidth]{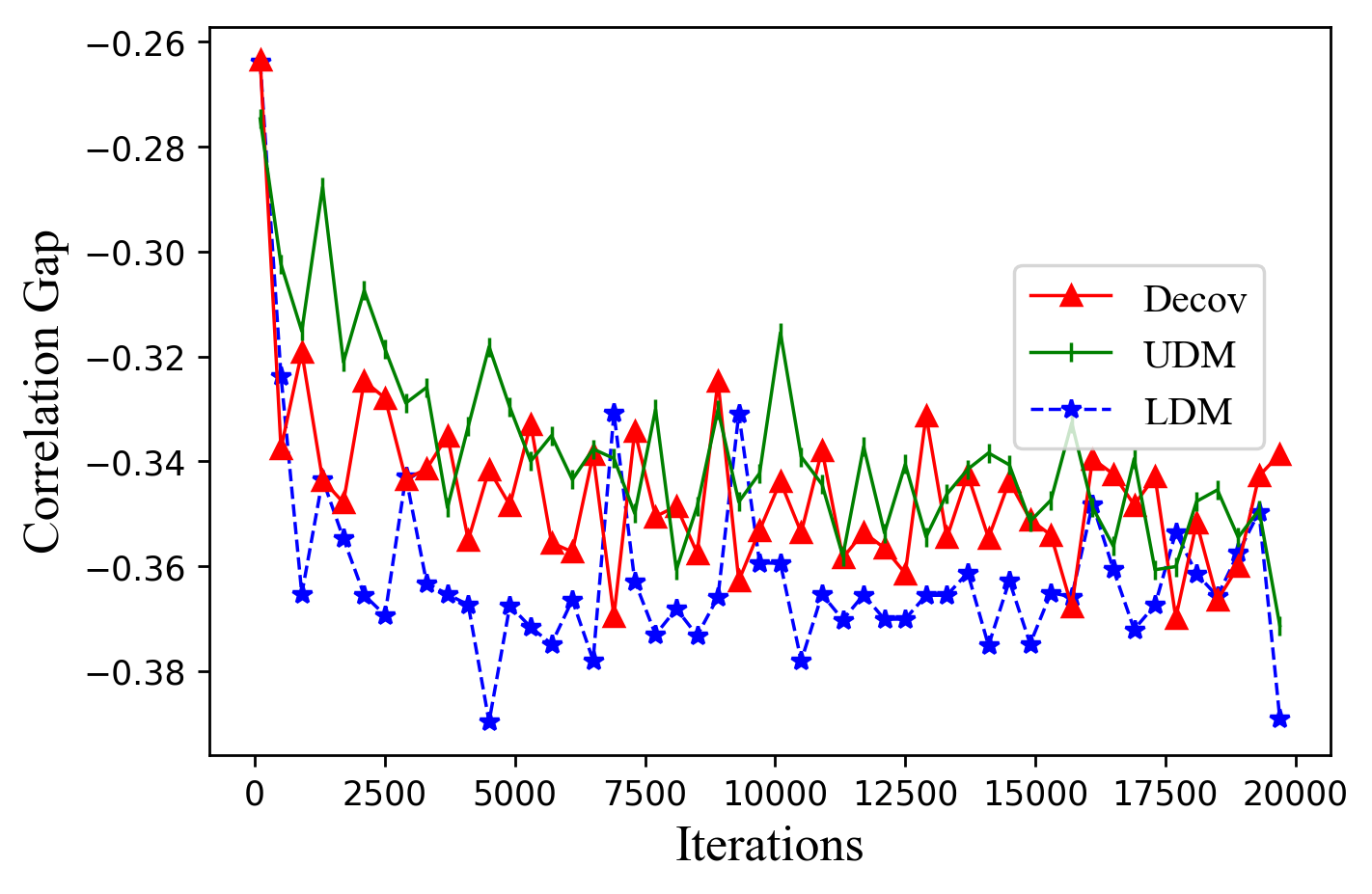} 
\caption{CIFAR-100}
\end{subfigure} 
\caption{ Comparing the correlation gap among the hidden units by different methods with increasing iteration numbers on (a) MNIST dataset; (b) CIFAR-10 dataset; (c) CIFAR-100 dataset}
\label{Fig1}
\end{figure*}

\begin{table*}[htpb]
\caption{Classification results of all methods on the MNIST dataset with fully connected networks and on the CIFAR-10 and CIFAR-100 datasets with a CNN. The best scores are in bold. The symbol ``$\dagger$"
means statistical improvement over all baselines}
\centering
\renewcommand\arraystretch{1.2}  
\begin{tabular}{|c|c|c|c|c|c|c|c|c|c|}  
\hline
  \multirow{2}{*}{Methods}&\multicolumn{3}{c|}{MNIST}  &   \multicolumn{3}{c|}{CIFAR-10}      & \multicolumn{3}{c|}{CIFAR-100} \\ \cline{2-10}
                           
             & Train   & Test  & Train - Test             & Train   & Test  & Train - Test    & Train    & Test     & Train-Test   \\
                          \hline
\hline
NONE         & 0.702 & 0.679  &0.023              & 0.978      & 0.740  &0.238                       &  0.949           &0.376          & 0.573     \\\hline 
Dropout      & 0.670     & 0.650  &0.020          & 0.980      & 0.755  &0.226                        &  0.925           & 0.436         & 0.489     \\\hline
Decov        & 0.699     & 0.674  &0.016          & 0.982       & 0.736  &0.246                         & 0.923           &0.379          & 0.544     \\\hline
UDM          & 0.703     & 0.675  &0.028          & 0.978       & 0.736  &0.242                          &  0.934           &0.381          & 0.553    \\\hline
LDM          & 0.691     & \textbf{0.680}  &\textbf{0.011}$^\dagger$     & 0.851 & \textbf{0.765}  &\textbf{0.086}$^\dagger$  & {0.610} &\textbf{0.440}& \bf{0.170}$^\dagger$     \\\hline
\end{tabular}
\label{tab:booktabs}
\end{table*}

\paragraph{Regularizer Comparisons.} We first checked whether LDM can reduce LDiversity. For a fair comparison, we did not investigate the value of LDiversity directly but checked the difference between the class-independent correlation and the class-conditional correlation, named the correlation gap, whose definition is given below. Given any pair of hidden units $i$ and $j$ $(1\leq i \leq m,1\leq j \leq m )$, the covariance between them is
\begin{equation}
C_{ij} = \frac{1}{n}\sum_{k=1}^{n}{(h_i^k-\mu_i)(h_j^k-\mu_j)},
\end{equation}
where $\mu_i=\frac{1}{n}\sum_{k}{h_i^k}$ is the sample mean of the activations of hidden unit $i$ over all the samples. Then the class-independent correlation is defined as 
\begin{equation}
\label{eq20}
Corre = \frac{1}{m(m-1)}\left(\sum_i \sum_j\frac{ |C_{ij}|}{\sqrt{C_{ii}C_{jj}}}-m\right).
\end{equation}
Conformably, the class-conditional correlation is the expectation of correlation over the class labels, which is estimated by using Eq. (\ref{eq20}) to obtain the respective correlations of given labels first and combining them according to the proportions of samples of each class as weights. 
Moreover, we did not record the correlation gap of dropout or NONE method since they were not designed to enforce the diversity.

Every experiment was repeated five times. The average results are shown in Fig. 3(a), from which we can see that as the number of iterations increases, the correlation gap of LDM is smaller than that of the other methods, indicating that LDM encourages the reduction of LDiversity. 

Further, we examined the classification performance of these methods, where the generalization capacity is evaluated by the difference between the training accuracy and test accuracy. The results are shown in the left part of Table 1. From Table 1, we can observe that LDM has the best test accuracy as well as a minimal  accuracy gap. We note that the other regularization methods differ from LDM mainly in the absence of the inductive-bias term. Particularly  for UDM, its regularizer is exactly the same as the first term of LDiversity. It is reasonable to attribute the good performance of LDM to the effect of the additional inductive-bias term in the LDiversity. Additionally, it is noteworthy that UDM achieves the worst performance on the accuracy gap even compared with the NONE method. The reason may be that over-disentanglement of the features may destroy the valuable local clustering information and weaken the learning ability of the networks. However, the use of regularization methods in the term of LDiversity can avoid such problems.

\subsection{Experiments on convolutional neural networks}

\paragraph{Dataset.} The experiments were conducted on the CIFAR-10/100 datasets. The CIFAR-10 dataset consists of 60000 32x32 color images in 10 classes, with 50000 training images and 10000 test images. CIFAR-100 is similar to CIFAR-10: it has 100 classes 6000 images, with per class 600 images.
\paragraph{Method Settings.} We compared these methods on the CIFAR10-quick architecture, which contains 3 convolutional layers followed by a fully connected layer with 64 hidden units and a softmax layer. Since the different features in the obtained representation by CNN may come from the same weights, i.e., they share the same mappings, while LDM measures the diversity among different mappings, we applied LDM to the last pool layer and regarded the resampled filters in this layer as mappings (see the Section Regularization Method). For fair comparison, we also apply Decov and UDM to this layer because they also used some diversity measure as the regularizer. Dropout is applied to the fully connected layer.
With exception of setting the number of iterations  to 20000, other training settings including the architectures of the two auxiliary networks were set to be the same as those in the fully connected neural networks.      
\paragraph{Regularizer Comparisons.}
We also investigated the changes in the values of the correlation gap with increasing numbers of iterations. The average results of all the methods over 5 trials are shown in Figs. 3(b) and (c), where only the results obtained by LDM, Decov and UDM are recorded. We can see that on both the CIFAR-10 and CIFAR-100 datasets, LDM achieves smaller correlation gaps than the other methods, which further confirms that LDM is an effective approach to enforce the diversity among the hidden units while strengthening the local clustering feature of different classes. 

The average classification accuracy of the examined methods on the CIFAR-10 and CIFAR-100 datasets are presented in the right two parts of Table 1. On the CIFAR-10 dataset, we observe that LDM outperforms the other methods on test accuracy and has a minimal train-test accuracy gap. In particular, LDM shows a 0.1 improvement on test accuracy and an approximately 0.14 improvement on the accuracy gap compared to UDM, which may be due to the use of the inductive-bias term in LDM.

We also observe similar results  on the CIFAR-100 dataset, justifying our hypothesis that the diversity measure should reflect the inductive-bias information. 

Finally, we tested the influence of hyperparameters on experimental performance and record the results on Table 2 and 3. From Table 2 and 3 we can find that LDM achieves the best performance when $\lambda = 0.7$ or $0.5$; moreover, as the value of $\lambda$ increases, it plays a larger role in LDM and controls the accuracy gap to be smaller, which justifying our point that the decrease in LDiversity generally improves the generalization capacity.
\begin{table}[tph]
\caption{Results with different value of $\lambda$ on dataset CIFAR-10 .}
\label{sample-table}
\begin{center}
\begin{small}
\begin{tabular}{cccccc}
\toprule
         & 0.1 & 0.3 & 0.5  & 0.7 &0.9 \\
\midrule
Train     & 0.982  & 0.973 & 0.923  & 0.851 & 0.826 \\
Test      & 0.743  &0.755  & 0.767  & 0.765 & 0.756\\
Train - test  & 0.217 &  0.239 &0.156 &0.086 &0.07  \\
\bottomrule
\end{tabular}
\end{small}
\end{center}
\end{table}

\begin{table}[tph]
\caption{Results with different value of $\lambda$ on dataset CIFAR-100 .}
\label{sample-table}
\begin{center}
\begin{small}
\begin{tabular}{cccccc}
\toprule
         & 0.1 & 0.3 & 0.5  & 0.7  \\
\midrule
Train     & 0.946  & 0.901 & 0.813  & 0.610  \\
Test      & 0.391  &0.401  & 0.432  & 0.440 \\
Train - test  & 0.555 & 0.5 &0.381 &0.170  \\
\bottomrule
\end{tabular}
\end{small}
\end{center}
\end{table}

\section{Conclusion}
In this paper, by investigating the upper bound of the generalization error from an information perspective, we found that we can naturally derive a measure of the diversity among the hidden units of DNNs, which differs from the other measures because it contains an inductive-bias term. Based on this insight, we designed a regularization method using the diversity measure as the regularizer. Our experiments verified the effectiveness of the proposed method and provided empirical evidence for the validity of our approach. 
\small
\bibliographystyle{named}

\end{document}